\documentclass[11pt]{article}

\usepackage{arxiv}

\usepackage{times}
\usepackage{latexsym}

\usepackage[T1]{fontenc}

\usepackage[utf8]{inputenc}
\usepackage{amsmath}
\usepackage{amssymb}
\usepackage{booktabs}

\usepackage{microtype}

\usepackage{inconsolata}

\usepackage{graphicx}
\usepackage{listings}
\usepackage{tcolorbox}
\usepackage{float}

\usepackage{multirow}

\usepackage{amsthm}
\usepackage{subcaption}
\newtheorem{proposition}{Proposition}

\usepackage[dvipsnames,table]{xcolor}

\usepackage{url}
\usepackage{natbib}
\newcommand{\ourmethod}{SC-SDPO}
\usepackage[colorlinks, citecolor=blue, linkcolor=blue]{hyperref}
\usepackage{cleveref}

\title{Restoring the Sweet Spot: Pass-Rate Weighted Self-Distillation for LLM Reasoning}

\author{
  \textbf{Zehao Liu}, \textbf{Yuanpu Cao}, \textbf{Jinghui Chen}, \textbf{Vasant G. Honavar} \\
  College of Information Sciences and Technology\\
  Pennsylvania State University \\
  \texttt{\{zml5418, ymc5533, jzc5917, vuh14\}@psu.edu}
}

\begin{document}
\maketitle
\begin{abstract}
Self-Distillation Policy Optimization (SDPO) provides dense token-level credit assignment for reinforcement learning with large language models by leveraging the model’s own feedback-conditioned predictions as a self-teacher.
Unlike GRPO, however, whose group-relative advantage naturally concentrates learning on a \emph{sweet spot} of intermediate-difficulty questions, SDPO’s KL-based advantage lacks an implicit notion of difficulty awareness.

We analyze this gap through the lens of GRPO’s advantage normalization. Extending the learnability framework to normalized rewards, we show that normalization absorbs the variance term $p(1-p)$, equalizing leading-order learnability across questions and leaving $\sqrt{p(1-p)}$ as the sole residual scaling factor in the per-question gradient.
This analysis yields a simple prescription: weight each question’s SDPO loss by $[\hat{p}(1-\hat{p})]^{1/2}$, resulting in \ourmethod{}, a scale-consistent variant of SDPO.

The proposed weights are obtained as a zero-cost byproduct of on-policy rollouts with batch-adaptive normalization, inducing an implicit curriculum that dynamically tracks the model’s evolving competence.
Experiments on scientific reasoning and tool-use benchmarks demonstrate that \ourmethod{} consistently improves over SDPO, yielding gains of +3.2/+4.3 (mean@16/maj@16) on Qwen3-8B and +1.8/+3.0 on OLMo-3-7B, while preserving stable training dynamics throughout optimization.
\end{abstract}

\section{Introduction}
Large language models (LLMs) are increasingly post-trained with reinforcement learning in verifiable domains such as code and mathematics~\citep{guo2025deepseek, jaech2024openai, team2025kimi, olmo2025olmo}. The dominant approach, reinforcement learning with verifiable rewards (RLVR), trains models using sparse binary outcome rewards. More recently, on-policy distillation (OPD) has emerged as a complementary paradigm that provides dense, token-level supervision from a teacher model, achieving strong performance on reasoning-intensive tasks~\citep{deepseekai2026deepseekv4, xiao2026mimo, yang2025qwen3, zeng2026glm}. Despite its effectiveness, RLVR suffers from a fundamental credit-assignment bottleneck: a single scalar reward per attempt indicates whether the model was wrong, but not where or why~\citep{hubotter2026reinforcement}.

Self-Distillation Policy Optimization \citep[SDPO;][]{hubotter2026reinforcement} offers an elegant solution to this bottleneck. By conditioning the same model on privileged information—such as runtime errors or previously discovered solutions—SDPO constructs a self-teacher that provides dense, logit-level supervision without requiring an external model. This mechanism converts tokenized feedback into fine-grained credit assignment and substantially outperforms Group Relative Policy Optimization (GRPO)~\citep{guo2025deepseek, shao2024deepseekmath}.

However, while SDPO resolves GRPO’s credit-assignment bottleneck, it also sacrifices a useful property that GRPO possesses by construction. In GRPO, the group-relative advantage naturally concentrates learning on a \emph{sweet spot} of intermediate-difficulty questions, where the model sometimes succeeds and sometimes fails. By contrast, SDPO’s KL-based distillation loss does not implicitly encode difficulty awareness: its gradient budget is largely insensitive to pass rate, diluting the learning signal precisely where self-distillation is most informative.

Prior work has formalized the role of difficulty in RLVR. \citet{bae2026online} show, under raw binary rewards, that policy improvement is lower-bounded by the reward variance $p(1-p)$. However, this analysis does not account for GRPO’s advantage normalization. We extend their framework to the normalized setting and show that normalization absorbs the $p(1-p)$ factor, equalizing leading-order learnability across questions. The residual per-question gradient magnitude instead scales as $\sqrt{p(1-p)}$ rather than $p(1-p)$ (\cref{sec:gap}).

This analysis directly motivates \ourmethod{}, which restores the sweet spot to SDPO by scaling each question’s loss contribution by $w = [\hat{p}(1-\hat{p})]^{1/2}$. The weighting automatically suppresses all-correct and all-incorrect groups while peaking at $\hat{p} = 0.5$. The weights are recomputed from the current policy’s rollouts at every training step using batch-adaptive normalization, forming an implicit curriculum that tracks the model’s evolving competence at zero additional rollout cost.

We evaluate \ourmethod{} on scientific reasoning and tool-use tasks. Our main findings are:
\begin{itemize}
\item \ourmethod{} improves over SDPO by +3.2/+4.3 (mean@16/maj@16) on Qwen3-8B and +1.8/+3.0 on OLMo-3-7B across five tasks, demonstrating that difficulty-aware weighting and dense credit assignment are complementary.
\item The $\alpha = 1/2$ exponent consistently outperforms $\alpha = 1$, supporting the scale-consistent argument.
\item Batch-adaptive weighting outperforms static single-pass schemes while incurring zero additional rollout cost.
\end{itemize}

\section{Related Work}

\paragraph{On-policy self-distillation.}
On-policy distillation trains a student model on its own generations to match a teacher's distribution~\citep{agarwal2024policy, gu2024minillm, lu2025onpolicydistillation, song2026survey}. When no external teacher is available, self-distillation methods condition the model on privileged in-context information to construct a self-teacher. Representative approaches include SDPO~\citep{hubotter2026reinforcement}, SDFT~\citep{shenfeld2026self}, OPSD~\citep{zhao2026self}, OPSDC~\citep{sang2026policy}, and G-OPD~\citep{yang2026learning}. Several works further improve sample routing within this paradigm: SCOPE~\citep{zheng2026scope} assesses signal reliability to adaptively filter noisy teacher corrections and reinforce diverse reasoning paths, HDPO~\citep{ding2026hdpo} tackles sparse rewards on excessively hard tasks in GRPO by injecting privileged information to construct high-quality distillation targets, SRPO~\citep{li2026unifying} routes samples based on correctness, applying token-level distillation to errors and GRPO to successes, TRACE~\citep{wang2026trace} introduces a token-level routing framework that adaptively splits the training objective between self-distillation and GRPO, RLSD~\citep{yang2026self} addresses privileged information leakage by decoupling the update direction (determined by environment reward) from the update magnitude (modulated by a per-token teacher-student likelihood ratio), and SD-ZERO~\citep{he2026self} takes a two-phase approach, first training the model to revise its own outputs conditioned on binary rewards, then distilling the revision capability back into the generator via on-policy KL minimization. 
% These methods modulate the self-distillation signal at the \emph{token} or \emph{sample} level; none weights questions by the model's current solve rate to concentrate training on the difficulty sweet spot.

\paragraph{Difficulty-aware training.}
DAPO~\citep{yu2026dapo} discards homogeneous rollout groups in RLVR, and \citet{bae2026online} provide a theoretical foundation under raw binary rewards by showing that policy improvement is lower-bounded by the reward variance $p(1-p)$. GRPO-LEAD~\citep{zhang2025grpo} and MathForge~\citep{dai2026harder} reweight advantages by the difficulty within the RLVR framework.  In the on-policy distillation setting, PACED~\citep{xu2026paced} derives a Beta-kernel weight $w(p) = p^\alpha(1-p)^\beta$ from the signal-to-noise ratio structure of distillation gradients, and demonstrates its effectiveness in both teacher-forced distillation and on-policy self-distillation. Our work differs from PACED in two key respects: we adopt $\alpha = \beta = 1/2$ rather than $\alpha = \beta = 1$, justified by \cref{sec:gap}; and we compute weights on-policy at every step as a free byproduct of the existing rollout loop, whereas PACED estimates pass rates in a separate offline rollout pass.

\section{Preliminaries}
\label{sec:preliminary}
\subsection{Self-Distillation Policy Optimization}
We build upon Self-Distillation Policy Optimization \citep[SDPO;][]{hubotter2026reinforcement}, which leverages the in-context learning ability of a language model for credit assignment in reinforcement learning. We briefly review its key components below.

\paragraph{Self-Teacher.}
Given a question $x$, the policy $\pi_\theta$ generates a response $y \sim \pi_\theta(\cdot \mid x)$ and receives tokenized feedback $f$ from the environment (e.g., correct solutions). SDPO introduces a \textit{self-teacher} $\pi_\theta(\cdot \mid x, f)$, which is the same policy additionally prompted with the feedback $f$. Since the self-teacher observes this extra information in context, it can retrospectively identify mistakes in the student's original attempt $y$ and adjust its next-token predictions accordingly. Crucially, compared with On-Policy Distillation\citep[OPD;][]{agarwal2024policy}, this mechanism requires no external teacher model.

\paragraph{Training Objective.}
SDPO minimizes a KL-divergence-based distillation loss between the student and the self-teacher at each token position $t$:
\begin{equation}
    \mathcal{L}_{\text{SDPO}}(\theta) :=  \sum_{t} \mathrm{KL}\big(\pi_\theta(\cdot \mid x, y_{<t}) 
     \parallel \mathrm{sg}(\pi_\theta(\cdot \mid x, f, y_{<t}))\big),
\label{eq:sdpo_loss}
\end{equation}
where $\text{sg}(\cdot)$ denotes the stop-gradient operator, which prevents gradients from flowing through the teacher.

\paragraph{Dense Credit Assignment.}
The gradient of $\mathcal{L}_{\text{SDPO}}$ can be interpreted as a policy gradient with per-token, per-logit advantages. Under the reverse KL formulation, the advantage takes the form:
\begin{equation}
    A^{\text{SDPO}}_{i,t}(\hat{y}_{i,t}) = \log \frac{\pi_\theta(\hat{y}_{i,t} \mid x, f_i, y_{i,<t})}{\pi_\theta(\hat{y}_{i,t} \mid x, y_{i,<t})},
    \label{eq:sdpo_advantage}
\end{equation}
where $\hat{y}_{i,t}$ ranges over the vocabulary at position $t$. This contrasts with GRPO \citep{shao2024deepseekmath, guo2025deepseek}, whose advantage is constant across all tokens within a rollout. In practice, the KL divergence is approximated by retaining only the top-$K$ logits of the student distribution to reduce memory overhead. Additionally, the symmetric Jensen--Shannon divergence can replace the KL divergence in Eq.~\eqref{eq:sdpo_loss} to further stabilize training. When the symmetric Jensen--Shannon divergence replaces the KL in Eq.~\eqref{eq:sdpo_loss}, the per-token advantage is given as:
\begin{equation}
    A_{\mathrm{JSD}}(\hat{y}_{i,t}) = \frac{1}{2} \log \frac{\pi_\theta(\hat{y}_{i,t})}{M(\hat{y}_{i,t})},
\end{equation}
where $M(\hat{y}_{i,t}) = \frac{1}{2}(\pi_\theta(\hat{y}_{i,t}) + \pi_\theta(\hat{y}_{i,t}|f))$ is the mixture. The derivation is provided in \cref{app:jsd_advantage}.

\subsection{GRPO and the Reward Sweet Spot}
In GRPO~\citep{guo2025deepseek, shao2024deepseekmath}, for each question $x$, a group of $G$ rollouts $\{y_i\}_{i=1}^G$ is sampled from the current policy, and each receives a binary reward $r_i \in \{0, 1\}$. The per-rollout estimated advantage is computed by normalizing the reward relative to the group: 
\begin{equation}
    \hat{A}_i^{\text{GRPO}} = \frac{r_i - \bar{r}}{\mathrm{std}(\{r_i\})},
    \label{eq:grpo_adv}
\end{equation}
where $\bar{r} = \frac{1}{G}\sum_{i=1}^G r_i$ is the group mean. When all rollouts succeed ($\bar{r}=1$) or all fail ($\bar{r}=0$), the advantages collapse to zero and no learning occurs. 
\citet{dai2026harder} formalize the total update magnitude for a single question as (Theorem 1):
\begin{equation}
\sum_{i=1}^{G}|\hat{A}_{i}^{\text{GRPO}}| 
= \sum_{i=1}^{G}\left|\frac{r_i - \bar{r}}{\mathrm{std}(\{r_i\})}\right| = 2G\sqrt{p(1-p)}.
\label{eq:sweet_spot_magnitude}
\end{equation}
\citet{bae2026online} formalize the learnability bound under the
raw (unnormalized) binary reward $r_{\mathrm{acc}}
\in \{0,1\}$, proving that the expected policy improvement
$J(\theta') - J(\theta)$ is lower-bounded by the within-group reward variance:
\begin{equation}
    J(\theta') - J(\theta)
    \;\geq\; \frac{1}{2\beta^2} \cdot \mathrm{Var}[r_{\mathrm{acc}}]
    = \frac{1}{2\beta^2} \cdot p(1-p),
    \label{eq:reward_variance}
\end{equation}
where $\beta > 0$ is the KL regularization temperature. Since $p(1-p)$ is a concave
quadratic that attains its unique maximum at $p = 0.5$,
% \begin{equation}
%     \arg\max_{p \in [0,1]}\; p(1-p) = 0.5,
%     \label{eq:sweet_spot}
% \end{equation}
the learning signal is strongest when roughly half the rollouts are correct, defining a \textit{reward sweet spot} at intermediate difficulty. We note that this conclusion rests on the raw reward prior to any normalization. As we show in \cref{sec:gap}, GRPO's group-relative normalization equalizes the leading-order learnability across all non-degenerate questions, broadening the sweet spot from a single point at $p = 0.5$ to the entire interval $p \in (0, 1)$; the residual $p$-dependence in the per-step gradient magnitude is $\sqrt{p(1-p)}$
rather than $p(1-p)$.

\section{Gap Between SDPO and GRPO}
\label{sec:gap}
 
The two frameworks offer complementary strengths that can be seen by comparing their per-question gradient contributions.
Let $\{x_j\}_{j=1}^{M}$ be the questions in a training batch, with $\mathcal{G}_j$ denoting the set of rollouts sampled from $\pi_\theta$ for question $x_j$. 
Both GRPO and SDPO can be expressed in a unified policy gradient form:
\begin{equation}
    \nabla_\theta \mathcal{L}_j = \sum_{i \in \mathcal{G}_j} \sum_{t} A_{i,t} \cdot \nabla_\theta \log \pi_\theta(y_{i,t} \mid x_j, y_{i,<t}).
    \label{eq:unified_pg}
\end{equation}
We decompose the advantage into two complementary factors---a \emph{question-level} component $A_j^{\text{q}}$ reflecting difficulty awareness, and a \emph{token-level} component $A_{i,t}^{\text{tok}}$ providing dense credit assignment:
\begin{equation}
    A_{i,t} = A_j^{\text{q}} \cdot A_{i,t}^{\text{tok}}.
    \label{eq:advantage_decompose}
\end{equation}
Under this decomposition, the two methods each capture only one factor:
\begin{align}
    \text{GRPO:} \quad &A_j^{\text{q}} = g(\hat{p}_j), \quad A_{i,t}^{\text{tok}} = 1, \label{eq:grpo_decompose} \\
    \text{SDPO:} \quad & A_j^{\text{q}} = 1, \quad A_{i,t}^{\text{tok}} = \log \frac{\pi_\theta(\hat{y}_{i,t} \mid f_i)}{\pi_\theta(\hat{y}_{i,t})}. \label{eq:sdpo_decompose}
\end{align}
where $g(\hat{p}_j)$ denotes the effective question-level signal strength of GRPO, and we suppress the shared conditioning on $(x, y_{i,<t})$ in $A_{i,t}^{\text{tok}}$ for brevity.
\paragraph{What each method lacks.} 
GRPO's group-relative advantage $\frac{r_i - \bar{r}}{\mathrm{std}(\{r_i\})}$ implicitly modulates the question-level signal---its aggregate magnitude scales as $\sqrt{\hat{p}_j(1-\hat{p}_j)}$, peaking at $\hat{p}_j = 0.5$ and vanishing for all-correct or all-incorrect groups---but assigns a uniform advantage across all tokens within a rollout ($A_{i,t}^{\text{tok}} = 1$). 
SDPO provides dense, per-token credit assignment through the self-teacher log-ratio but treats all questions equally ($A_j^{\text{q}} = 1$), losing the difficulty-aware modulation that GRPO possesses by construction. 
We verify this empirically in Figure \ref{fig:kl_analysis} (Appendix): the per-token SDPO advantage magnitude remains largely constant across pass rates, confirming that SDPO's distillation loss does not implicitly recover question-level signal modulation.

\paragraph{Why the ideal scales as $\sqrt{p(1-p)}$, not $p(1-p)$.}

The variance-based scaling $p(1-p)$ originates from the analysis of \citet{bae2026online}, who derive their lower bound (Eq.~\ref{eq:reward_variance}) under the raw binary reward $r_{\mathrm{acc}} \in \{0,1\}$ without accounting for advantage normalization. However, GRPO does not use the raw reward as its advantage;
instead, it applies group-relative normalization (Eq.~\ref{eq:grpo_adv}), yielding the standardized advantage $\hat{A}_i = (r_i - \bar{r}) / \mathrm{std} (\{r_i\})$.
 
We extend the theoretical framework of \citet{bae2026online} to this normalized setting. Specifically, we apply their cumulant generating function (CGF) analysis
(Proposition~6.1 in \citealt{bae2026online}) to the standardized reward
$\tilde{r} := (r - p)/\sqrt{p(1-p)}$.
The main result (Proposition~\ref{prop:normalized_cgf} in \cref{sec:normalized_bound}) shows that the leading term of the learnability bound becomes a constant $1/(2\beta^2)$ independent of~$p$:
\begin{equation}
    D_{\mathrm{KL}}\!\left(\pi_{\mathrm{init}} \,\|\,
    \tilde{\pi}^*\right)
    = \frac{1}{2\beta^2}
    + O(\beta^{-3}),
    \label{eq:normalized_kl}
\end{equation}
where $\tilde{\pi}^*$ is the optimal policy under the normalized reward. The normalization absorbs the $p(1-p)$ factor from the raw-reward bound, equalizing leading-order learnability across all non-degenerate questions. The only remaining $p$-dependence is GRPO's per-question gradient magnitude $\sqrt{p(1-p)}$ (Eq.~\ref{eq:sweet_spot_magnitude}), which the ideal objective adopts as its question-level factor.

\paragraph{An ideal objective.}
Given that the natural question-level scaling is $\sqrt{p(1-p)}$, the decomposition in Eq.~\eqref{eq:advantage_decompose} suggests a natural ideal objective that combines the strengths of both frameworks: dense token-level credit assignment from SDPO together with GRPO's question-level difficulty modulation.
Concretely, the ideal per-question gradient takes the form:
\begin{equation}
    \nabla_\theta \mathcal{L}_j^{\,\mathrm{ideal}}
    = \underbrace{[\hat{p}_j(1-\hat{p}_j)]^{1/2}}_
      {A_j^{q,\star}:\;\text{question-level}}
    \;\cdot
    \underbrace{\sum_{i \in \mathcal{G}_j} \sum_t
      A_{i,t}^{\mathrm{SDPO}} \cdot
      \nabla_\theta \log \pi_\theta
      (y_{i,t} \mid x_j, y_{i,<t})}_
      {\text{token-level (SDPO)}},
    \label{eq:ideal}
\end{equation}
where the question-level factor $A_j^{q,\star} =
[\hat{p}_j(1-\hat{p}_j)]^{1/2}$ reproduces the natural
scaling that GRPO obtains implicitly through group-relative
advantage normalization.

\section{Method}
\label{sec:method}
The ideal objective (Eq.~\ref{eq:ideal}) prescribes a specific question-level factor but leaves open how to implement it within SDPO's existing training loop. We show that a single scalar modification suffices.
As shown in Eq.~\eqref{eq:advantage_decompose}, the advantage decomposes into a question-level factor $A_j^{\text{q}}$ and a token-level factor $A_{i,t}^{\text{tok}}$. We propose to set $A_j^{\text{q}} = \textcolor{RoyalBlue}{\bar{w}_j}$, a per-question weight derived from on-policy pass rates, yielding the Scale-Consistent SDPO (\ourmethod{}) gradient:
\begin{equation}
    \nabla_\theta \mathcal{L}_j = \sum_{i \in \mathcal{G}_j} \sum_{t} \textcolor{RoyalBlue}{\bar{w}_j} \cdot A_{i,t}^{\text{SDPO}} \cdot \nabla_\theta \log \pi_\theta(y_{i,t} \mid x_j, y_{i,<t}).
    \label{eq:wsdpo_grad}
\end{equation}
The modification is minimal---a single scalar multiplier per question---yet it restores the question-level signal modulation that SDPO lacks. We detail the construction of $\textcolor{RoyalBlue}{\bar{w}_j}$ in two parts: the weighting function (\cref{sec:weighted_obj}), and the batch-adaptive normalization (\cref{sec:normalization}).
 
\subsection{Weighted SDPO Objective}
\label{sec:weighted_obj}
 
Of the $|\mathcal{G}_j|$ rollouts, let $k_j$ succeed, giving an empirical pass rate $\hat{p}_j = k_j / |\mathcal{G}_j|$. The \ourmethod{} objective is:
\begin{equation}
    \mathcal{L}_{\text{\ourmethod}}(\theta) :=  \sum_{j=1}^{M} \textcolor{RoyalBlue}{\bar{w}_j} \sum_{i \in \mathcal{G}_j} \sum_{t} \mathrm{KL}\big( \pi_\theta(\cdot \mid x_j, y_{i,<t}) \big\|\, \mathrm{sg}\big(\pi_\theta(\cdot \mid x_j, f_i, y_{i,<t})\big) \big),
\label{eq:weighted_sdpo}
\end{equation}
where $\textcolor{RoyalBlue}{\bar{w}_j}$ is a normalized, question-level weight derived from $\hat{p}_j$. The inner two sums are identical to the original SDPO loss; the only change is the outer multiplication by $\textcolor{RoyalBlue}{\bar{w}_j}$.
 
The unnormalized weight takes the form:
\begin{equation}
    w_j = \bigl[\hat{p}_j(1 - \hat{p}_j)\bigr]^{\alpha}, \quad \alpha > 0.
    \label{eq:raw_weight}
\end{equation}
The term $\hat{p}(1-\hat{p})$ is the Bernoulli variance of the rollout outcomes and equals zero when $\hat{p}\in\{0,1\}$, reaching its maximum at $\hat{p}=0.5$. Questions where the policy sometimes succeeds and sometimes fails---the regime where contrastive self-distillation is most informative---receive the largest weight. The exponent $\alpha$ controls how sharply the weight concentrates around this sweet spot. We set $\alpha = 1/2$ by default, as justified by the 
scale-consistency analysis in \cref{sec:gap}.

\subsection{Batch-Adaptive Normalization}
\label{sec:normalization}
 
As the policy evolves and more questions migrate to the extremes ($\hat{p} \to 1$), the number of non-degenerate questions $|\mathcal{S}|$ and the raw weight $w_j$ sum both change, potentially destabilizing the effective learning rate. 
We apply a simple within-batch normalization. Let $\mathcal{S} = \{j : w_j > 0\}$ be the set of questions with non-degenerate pass rates ($\hat{p}_j \notin \{0,1\}$). The normalized weight is:
\begin{equation}
    \textcolor{RoyalBlue}{\bar{w}_j} = 
    \begin{cases}
        \displaystyle \frac{w_j}{\frac{1}{|\mathcal{S}|}\sum_{l \in \mathcal{S}} w_l}, & \text{if } j \in \mathcal{S}, \\[6pt]
        0, & \text{otherwise}.
    \end{cases}
    \label{eq:normalization}
\end{equation}
This ensures that the mean weight over informative questions is always 1, preventing the effective learning rate from drifting as the batch composition changes during training. It does not affect the relative weight distribution and hence does not alter the scale-consistency to the ideal objective.
 
Because SDPO is an on-policy algorithm that already samples $G$ rollouts per question and evaluates them at every training step, the pass rate $\hat{p}_j$ is a \emph{free byproduct} of the existing pipeline. Computing $w_j$ and $\textcolor{RoyalBlue}{\bar{w}_j}$ adds $\mathcal{O}(M)$ scalar operations---negligible compared to the forward and backward passes. Unlike single-pass estimation methods such as PACED~\citep{xu2026paced}, which compute pass rates once before training and freeze them, our weights are recomputed at every step and therefore automatically track the evolving policy: a previously unsolvable question that enters the sweet spot as training progresses will receive increasing weight, while mastered questions are naturally phased out. This yields an implicit curriculum without manual scheduling.
 
\paragraph{Summary.}
\ourmethod{} modifies SDPO by a single per-question scalar multiplication. 
The weight $\textcolor{RoyalBlue}{\bar{w}_j}$ restores the question-level gradient modulation that GRPO obtains implicitly, 
concentrating training on the difficulty sweet spot while preserving SDPO's dense, per-token credit assignment. 
The implementation is shown in Figure~\ref{fig:code}.

\begin{figure}[t]
\centering
\begin{minipage}{0.47\textwidth}
\begin{lstlisting}[
  language=Python,
  basicstyle=\ttfamily\small,
  keywordstyle=\color{RoyalBlue}\bfseries,
  commentstyle=\color{gray},
  stringstyle=\color{OliveGreen},
  columns=fullflexible,
  frame=single,
  xleftmargin=6pt,
  xrightmargin=6pt,
  aboveskip=6pt,
  belowskip=6pt,
]
# --- SC-SDPO: 6-line modification to SDPO ---
# Input: rewards (B,), uids (B,), alpha=0.5
 
# Step 1: per-question pass rate
for uid in unique(uids):
    p[uid] = rewards[uids == uid].mean()
 
# Step 2: std-scaling weight
w_raw = {uid: (p * (1 - p)) ** alpha 
     for uid, p in p.items()}
 
# Step 3: batch-adaptive normalization
w_nz = [w for w in w_raw.values() if w > 0]
w = {uid: w / mean(w_nz) if w > 0 else 0
     for uid, w in w_raw.items()}
 
# Step 4: apply to SDPO loss
per_token_loss = per_token_loss * w[uids] 
\end{lstlisting}
\end{minipage}
\caption{Implementation of \ourmethod{}. The entire modification reduces to computing a per-question scalar weight from on-policy rollouts (Steps 1--3) and a single element-wise multiplication on the existing SDPO loss (Step 4). No architectural changes, additional rollouts, or hyperparameter tuning beyond the default $\alpha = 0.5$ are required.}
\label{fig:code}
\end{figure}

\section{Experiment}
\subsection{Experimental Setting}
\label{sec:exp_setting}

\paragraph{Models and training.}
Following \citet{hubotter2026reinforcement}, we conduct experiments on Qwen3-8B~\citep{yang2025qwen3} and OLmo3-7B-Instruct~\citep{olmo2025olmo}, both in non-thinking mode. Qwen3-8B serves as our primary model; OLmo3-7B-Instruct
is used to verify generalization across model families.
For OLMo-3-7B, we report only the core comparison
between SDPO and \ourmethod{}
(\(\alpha\!=\!0.5\)) due to computational constraints;
the full ablation suite (PACED, Hard Filter, SRPO,
\(\alpha\!=\!1\)) is conducted exclusively on Qwen3-8B.
All runs are trained for 200 steps on 2 NVIDIA H200 GPUs using the verl framework~\citep{sheng2024hybridflow}. We report results in terms of training steps rather than wall-clock time (as in \citealp{hubotter2026reinforcement, li2026unifying}), since step-based evaluation is independent of hardware configuration and ensures reproducibility across different compute environments.
\paragraph{Data.}
We evaluate on two task families for which neither model has been explicitly fine-tuned: (1) \textbf{Science Q\&A} (Chemistry, Physics, Biology, Materials Science), using the undergraduate-level reasoning subsets (L3) from SciKnowEval~\citep{feng2024sciknoweval}; and (2) \textbf{Tool Use}, mapping tool-API specifications and user requests to correct tool calls, using ToolAlpaca~\citep{tang2023toolalpaca}. We adopt the same train-test splits as \citet{hubotter2026reinforcement}. Prompt templates are provided in \cref{app:prompt}.

\begin{table*}[!t]
\centering
\resizebox{0.98\textwidth}{!}{
\begin{tabular}{l|cc|cc|cc|cc|cc|cc}
\toprule
\rowcolor[HTML]{E9F3FE}  & \multicolumn{2}{c|}{\textbf{Chem.}} & \multicolumn{2}{c|}{\textbf{Phys.}} & \multicolumn{2}{c|}{\textbf{Bio.}} & \multicolumn{2}{c|}{\textbf{Mat.}} & \multicolumn{2}{c|}{\textbf{Tool Use}} & \multicolumn{2}{c}{\textbf{Avg.}} \\ 
\cline{2-13}
\rowcolor[HTML]{E9F3FE} \multirow{-2}{*}{\textbf{Method}} & mean & maj & mean & maj & mean & maj & mean & maj & mean & maj & mean & maj \\
\midrule \midrule
Qwen3-8B & 42.4 & 35.7 & 57.7 & 53.8 & 28.7 & 14.0 & 59.7 & 59.6 & 58.3 & 57.4 & 49.4 & 44.1 \\
\midrule
+ GRPO & 68.5 & 70.9 & 64.7 & 68.3 & 58.9 & 60.0 & 69.8 & 69.9 & \underline{67.3} & 67.6 & 65.8 & 67.3 \\
+ GRPO w/o Norm & 63.9 & 67.0 & 69.1 & 76.5 & 53.2 & 53.9 & 73.3 & 74.5 & 63.5 & 63.4 & 64.6 & 67.1 \\
+ SDPO & 78.7 & 79.2 & 75.8 & 78.5 & \underline{60.4} & \underline{61.9} & \underline{78.0} & \underline{78.6} & 65.1 & 64.8 & 71.6 & 72.6 \\
+ SDPO w/ PACED & 74.7 & 75.5 & 69.2 & 71.7 & 57.0 & 59.0 & 75.4 & 75.8 & 64.5 & 65.7 & 68.2 & 69.5 \\
+ SDPO w/ Hard Filter & 80.0 & 80.4 & \underline{76.4} & \underline{79.6} & 56.0 & 56.4 & 77.3 & 78.1 & 63.8 & 64.8 & 70.7 & 71.9 \\
+ SRPO & \underline{80.5} & \textbf{81.6} & 73.4 & 76.3 & 56.0 & 60.8 & 77.2 & \textbf{80.3} & \textbf{71.8} & \textbf{73.3} & \underline{71.8} & \underline{74.5} \\
+ \ourmethod{} ($\alpha$=1) & 80.0 & 79.5 & 75.4 & 77.3 & 58.1 & 56.3 & 77.1 & 78.3 & 66.2 & 66.9 & 71.3 & 71.7 \\
\midrule
\cellcolor[HTML]{DAE0FB}\textbf{+ \ourmethod{} ($\alpha$=0.5)} & \cellcolor[HTML]{DAE0FB}\textbf{80.6} & \cellcolor[HTML]{DAE0FB}\underline{81.3} & \cellcolor[HTML]{DAE0FB}\textbf{81.6} & \cellcolor[HTML]{DAE0FB}\textbf{84.9} & \cellcolor[HTML]{DAE0FB}\textbf{65.4} & \cellcolor[HTML]{DAE0FB}\textbf{68.5} & \cellcolor[HTML]{DAE0FB}\textbf{79.3} & \cellcolor[HTML]{DAE0FB}\textbf{80.3} & \cellcolor[HTML]{DAE0FB}\underline{67.3} & \cellcolor[HTML]{DAE0FB}\underline{69.3} & \cellcolor[HTML]{DAE0FB}\textbf{74.8} & \cellcolor[HTML]{DAE0FB}\textbf{76.9} \\
\bottomrule
\end{tabular}}
\caption{Qwen3-8B results comparing mean@16 and maj@16 (\%). \textbf{Bold} and \underline{underlined} values indicate the best and second-best results, respectively.}
\label{tab:merged_results16}
% \end{table*}

% \begin{table*}[!t]
\vspace{0.5em}
\centering
\resizebox{0.98\textwidth}{!}{
\begin{tabular}{l|cc|cc|cc|cc|cc|cc}
\toprule
\rowcolor[HTML]{E9F3FE}  & \multicolumn{2}{c|}{\textbf{Chem.}} & \multicolumn{2}{c|}{\textbf{Phys.}} & \multicolumn{2}{c|}{\textbf{Bio.}} & \multicolumn{2}{c|}{\textbf{Mat.}} & \multicolumn{2}{c|}{\textbf{Tool Use}} & \multicolumn{2}{c}{\textbf{Avg.}} \\ 
\cline{2-13}
\rowcolor[HTML]{E9F3FE} \multirow{-2}{*}{\textbf{Method}} & mean & maj & mean & maj & mean & maj & mean & maj & mean & maj & mean & maj \\
\midrule \midrule
OLMo-3-7B & 22.1 & 9.5 & 33.3 & 30.0 & 13.4 & 8.0 & 36.6 & 30.9 & 43.4 & 45.6 & 29.7 & 24.8 \\
\midrule
+ SDPO & \underline{81.1} & \underline{81.4} & \underline{68.2} & \underline{68.7} & \underline{52.0} & \underline{52.0} & \underline{78.0} & \underline{80.0} & \underline{62.1} & \underline{65.0} & \underline{68.3} & \underline{69.4} \\
\midrule
\cellcolor[HTML]{DAE0FB}\textbf{+ \ourmethod{} ($\alpha$=0.5)} & \cellcolor[HTML]{DAE0FB}\textbf{81.5} & \cellcolor[HTML]{DAE0FB}\textbf{82.5} & \cellcolor[HTML]{DAE0FB}\textbf{71.1} & \cellcolor[HTML]{DAE0FB}\textbf{73.4} & \cellcolor[HTML]{DAE0FB}\textbf{55.9} & \cellcolor[HTML]{DAE0FB}\textbf{58.2} & \cellcolor[HTML]{DAE0FB}\textbf{79.1} & \cellcolor[HTML]{DAE0FB}\textbf{80.7} & \cellcolor[HTML]{DAE0FB}\textbf{62.9} & \cellcolor[HTML]{DAE0FB}\textbf{67.3} & \cellcolor[HTML]{DAE0FB}\textbf{70.1} & \cellcolor[HTML]{DAE0FB}\textbf{72.4} \\
\bottomrule
\end{tabular}}
\caption{OLMo-3-7B results comparing mean@16 and maj@16 (\%). \textbf{Bold} and \underline{underlined} values indicate the best and second-best results, respectively.}
\label{tab:olmo_results16}
\end{table*}

\paragraph{Methods.}
We compare 8 configurations:
(1) \textbf{GRPO}: the standard group-relative policy optimization baseline, using the same optimization setup as \citet{hubotter2026reinforcement};
(2) \textbf{GRPO w/o Norm}: GRPO without the standard-deviation normalization in the advantage, using raw centered rewards $r_i - \bar{r}$ directly;
(3) \textbf{SDPO}: the unweighted baseline;
(4) \textbf{SDPO w/ PACED}: pass-rate weights $w_j = \hat{p}_j(1-\hat{p}_j)$ estimated in a single offline rollout pass before training and frozen throughout, following \citet{xu2026paced};
(5) \textbf{SDPO w/ Hard Filter}: a binary filtering strategy that retains only questions with $0.2 \leq \hat{p}_j \leq 0.8$ and discards the rest, recomputed on-policy at each step;
(6) \textbf{SRPO}~\citep{li2026unifying}: uses the sample-level routing and dynamic-weighted SDPO mechanisms proposed in the original paper with default hyperparameters; all other settings (learning rate, batch size, number of rollouts, etc.) are kept consistent with SDPO and the other methods in this work;
(7) \textbf{\ourmethod{} ($\alpha{=}1$)}: batch-adaptive variance-based weighting $w_j = \hat{p}_j(1-\hat{p}_j)$ with on-policy recomputation and normalization at every step;
(8) \textbf{\ourmethod{} ($\alpha{=}0.5$)}: our proposed method, batch-adaptive scale-consistent weighting $w_j = [\hat{p}_j(1-\hat{p}_j)]^{1/2}$ with on-policy recomputation and normalization at every step.
All other hyperparameters (learning rate, batch size, number of rollouts, EMA rate) are held constant across configurations to isolate the effect of the weighting scheme.
Details of implementation and hyperparameter are given in \cref{sec:appendix_hyperparameters}.

\subsection{Main Results}
\label{sec:main_results}

Tables~\ref{tab:merged_results16} and~\ref{tab:olmo_results16} summarize the main results on Qwen3-8B and OLMo-3-7B, respectively.

\paragraph{\ourmethod{} achieves the best overall performance.}
On Qwen3-8B, \ourmethod{} ($\alpha{=}0.5$) achieves the highest average accuracy on both mean@16 (74.8\%) and maj@16 (76.9\%), outperforming the SDPO baseline by +3.2 and +4.3 points, respectively. The improvement is consistent across tasks, with particularly large gains on Physics (+5.8/+6.4) and Biology (+5.0/+6.6), where the pass rate distribution offers the broadest intermediate zone for difficulty-aware weighting to exploit (Figure~\ref{fig:passrate_dynamics}).

\paragraph{Dense credit assignment is essential.}
Both GRPO variants substantially underperform all SDPO-based methods, confirming that SDPO's dense, token-level credit assignment provides a fundamental advantage over GRPO's uniform per-token signal. Notably, removing the advantage normalization (GRPO w/o Norm) further degrades performance, consistent with our theoretical analysis in \cref{sec:gap}: normalization equalizes learnability across questions and is a key component of GRPO's implicit difficulty awareness. These results underscore that \ourmethod{}'s gains come from combining both ingredients---dense credit and difficulty modulation---rather than from either alone.

\paragraph{Scale-consistent weighting outperforms variance-based weighting.}
Comparing the two exponents, $\alpha{=}0.5$ consistently outperforms $\alpha{=}1$ across all tasks, with an average gap of +3.5 on mean@16 and +5.2 on maj@16. This validates our theoretical argument 
(\cref{sec:gap}): under GRPO's normalized rewards, the leading-order learnability is equalized across questions, and the residual per-question gradient magnitude scales as $\sqrt{p(1-p)}$; $\alpha{=}1$ over-suppresses this signal by an additional factor of $\sqrt{p(1-p)} \leq 0.5$.(Figure~\ref{fig:combined}, Right).

\paragraph{Static weighting and hard filtering underperform.}
PACED, which applies variance-based weights ($\alpha{=}1$) frozen at initialization, is the weakest method---falling 3.4/3.1 points below the unweighted SDPO baseline on average. This degradation is attributable to both the squared-scale suppression of $\alpha{=}1$ and the staleness of frozen weights, which become increasingly misaligned with the model's evolving competence (\cref{app:passrate_analysis}). Hard filtering (retaining $0.2 \leq \hat{p} \leq 0.8$) performs closer to SDPO but still trails \ourmethod{} by 4.1/5.0 points on average, confirming that soft, continuous weighting extracts more signal than binary thresholding.

\paragraph{Comparison with SRPO.}
SRPO achieves competitive results, ranking second on maj@16 (74.5\%) and obtaining the best scores on Tool Use (71.8/73.3) and Materials maj@16 (80.3). However, SRPO's sample-level routing mechanism introduces substantially higher gradient variance (Figure~\ref{fig:combined}, Right) and training instability compared to \ourmethod{}, which achieves a higher average with smoother optimization dynamics. Moreover, SRPO's gains on Tool Use---where pass rates are concentrated at the extremes (\cref{app:passrate_analysis})---suggest that its GRPO branch provides complementary signal in regimes where difficulty-aware weighting has limited leverage.

\paragraph{Cross-model generalization.}
On OLMo-3-7B (Table~\ref{tab:olmo_results16}), \ourmethod{} ($\alpha{=}0.5$) improves over SDPO by +1.8/+3.0 on mean@16/maj@16, with gains across all five tasks. This confirms that the benefit of pass-rate weighting is not specific to the Qwen model family.

\paragraph{Response length.}
\label{response_length_analysis}
\begin{figure}[H]
    \centering
    \includegraphics[width=0.45\linewidth]{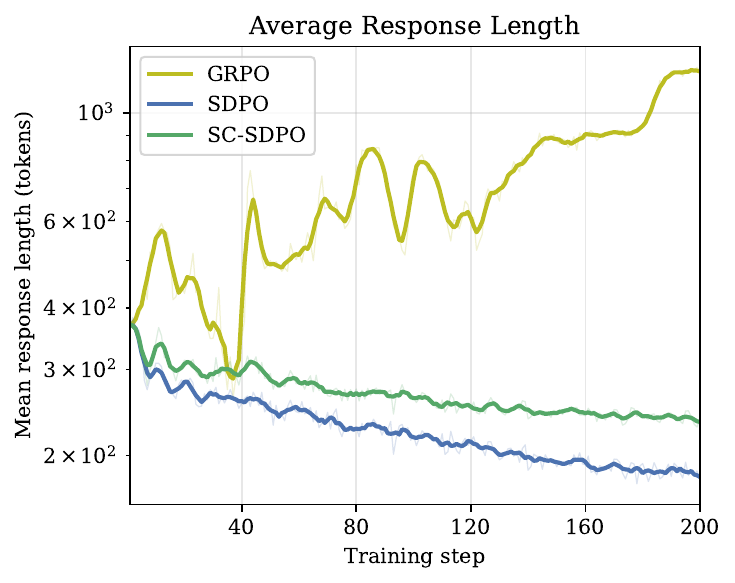}
    \caption{
        Average response length (in tokens) over 200 training steps, averaged across five datasets.
    }
    \label{fig:response_length}
\end{figure}

Figure~\ref{fig:response_length} shows that \ourmethod{} and SDPO exhibit a declining trend in response length as training progresses, consistent with prior findings that on-policy RL encourages more concise reasoning~\citep{hubotter2026reinforcement}. Notably, \ourmethod{} maintains consistently higher response lengths than the SDPO baseline throughout training, suggesting that pass-rate weighting steers the model toward more elaborate reasoning on intermediate-difficulty questions rather than collapsing to short, confident answers on already-mastered ones.

\begin{figure}[H]
    \centering
    \includegraphics[width=\linewidth]{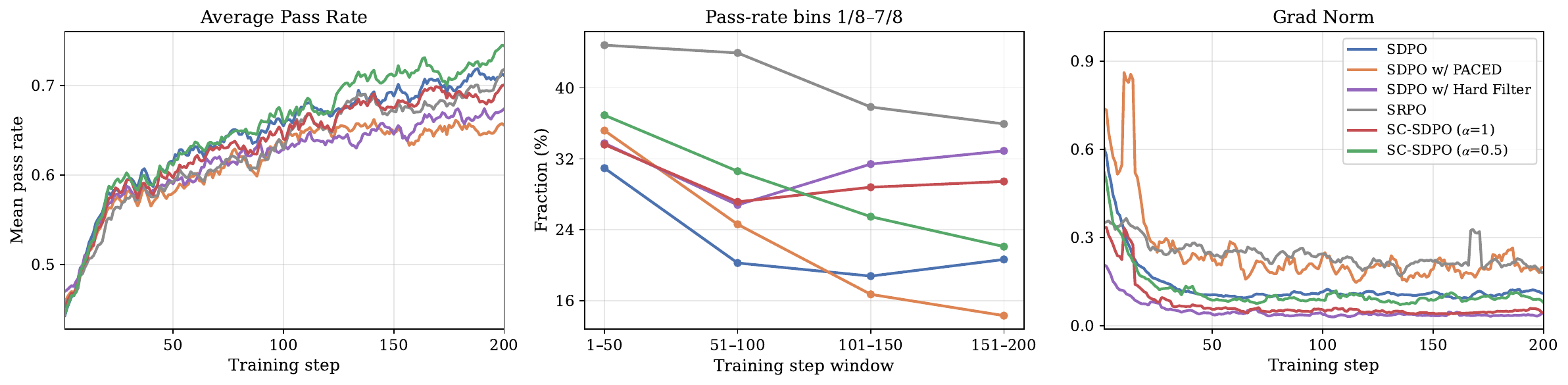}
    \caption{
Averaged training dynamics across five datasets. \textbf{(Left)} Mean pass rate per training step.
  \textbf{(Middle)} Fraction of training samples in pass-rate bins $\frac{1}{8}$--$\frac{7}{8}$ per 50-step
  window. \textbf{(Right)} Actor gradient norm per training step.
    }
    \label{fig:combined}
\end{figure}
\subsection{Training Dynamics Analysis}
\label{sec:training_dynamics}
Figure~\ref{fig:combined} presents three complementary views of the training dynamics.
Per-dataset breakdowns and extended analysis are provided in \cref{app:passrate_analysis}.

\paragraph{Pass rate progression (Left).} All methods exhibit rising mean pass rates as training progresses, reflecting the model's growing competence. \ourmethod{} ($\alpha{=}0.5$) achieves the highest mean pass rate in the later stages of training, indicating that concentrating gradient budget on the sweet spot accelerates overall skill acquisition. ToolUse is a notable exception where pass rates remain low across all methods due to the polarized difficulty distribution (cf.\ Figure~\ref{fig:passrate_dynamics} in the Appendix).

\paragraph{Sweet spot retention (Middle).} SRPO maintains the highest fraction of intermediate-pass-rate samples throughout training, which is expected given its routing design that preserves mixed-outcome groups. However, as shown in Table~\ref{tab:merged_results16}, this does not translate into the best average accuracy, suggesting that a large sweet spot is necessary but not sufficient---optimization stability and signal quality matter equally. Among non-routing methods, \ourmethod{} ($\alpha{=}0.5$) retains the highest intermediate fraction in early training (steps 1--50), when the learning frontier is broadest and difficulty-aware weighting has the greatest impact. As more questions are mastered, the fraction naturally declines across all methods; \ourmethod{} converges smoothly, consistent with the implicit curriculum effect of our batch-adaptive design.

\paragraph{Gradient stability (Right).} \ourmethod{} ($\alpha{=}0.5$) maintains a gradient norm profile closely aligned with the unweighted SDPO baseline, confirming that scale-consistent weighting reshapes the gradient distribution across questions without disrupting the overall optimization scale. This stability follows directly from the batch-adaptive normalization (Eq.~\ref{eq:normalization}), which preserves unit-mean weights. In contrast, PACED produces large gradient spikes in early training due to two compounding factors: its global normalization allows batch-level weight sums to fluctuate, and its frozen weights become misaligned with the model's evolving competence. SRPO exhibits persistently unstable gradients, suggesting that its sample-level routing mechanism introduces inconsistent gradient directions across steps. The variance-based weighting ($\alpha{=}1$) drives the gradient norm well below the baseline, starving the optimizer of learning signal, while the hard filter yields similarly suppressed gradients due to its reduced effective batch size.

\section{Conclusion}

We identified a fundamental gap between SDPO and GRPO: SDPO provides dense token-level credit assignment but lacks question-level difficulty awareness, while GRPO possesses the opposite strengths. By extending the learnability analysis of \citet{bae2026online} to GRPO's normalized rewards, we showed that advantage normalization absorbs the $p(1-p)$ factor from the raw-reward bound, leaving $\sqrt{p(1-p)}$ as the natural per-question gradient scaling. This analysis directly prescribed \ourmethod{}, which restores difficulty awareness to SDPO through a single scalar weight per question. Experiments on scientific reasoning and tool-use tasks confirmed that \ourmethod{} consistently outperforms SDPO and all baselines, with $\alpha = 1/2$ validating the scale-consistent scaling predicted by our theory.

\section{Limitations}

\paragraph{Dependence on pass-rate granularity.}
With $G = 8$ rollouts per question, the empirical pass rate $\hat{p}$ takes only nine discrete values $\{0, \tfrac{1}{8}, \ldots, 1\}$. The Beta-type weight $[\hat{p}(1-\hat{p})]^{1/2}$ smooths over this discretization, but fine-grained difficulty distinctions may be lost. Increasing $G$ improves resolution at the cost of additional generation compute.

\paragraph{Binary reward assumption.}
Our theoretical justification for $\alpha = 1/2$ relies on the Bernoulli structure of binary rewards. In domains with continuous or partial-credit rewards, the relationship between reward variance and gradient signal strength may differ, and the optimal exponent could change accordingly.

\paragraph{Task scope.}
We evaluate on scientific reasoning and tool-use tasks with verifiable answers. Extending pass-rate weighting to open-ended generation tasks (e.g., creative writing, summarization) would require a proxy competence signal such as reward-model scores, introducing additional noise and calibration challenges.

\paragraph{Model scale.}
Our experiments are conducted on 7--8B parameter models. While the method is architecture-agnostic, its effectiveness on substantially larger or smaller models remains to be validated. Following the scaling findings of SDPO~\citep{hubotter2026reinforcement}, we expect the benefit to grow with model scale, as stronger in-context learners produce more reliable self-teacher signals, but this hypothesis requires empirical confirmation.

\paragraph{Cross-model ablation coverage.}
Our full ablation suite is conducted only on Qwen3-8B.
OLMo-3-7B is limited to the core SDPO
vs.\ \ourmethod{} comparison, which suffices to
confirm cross-family generalization of the
proposed weighting but leaves open whether the
relative ranking among baselines transfers
across architectures. Extending the complete
ablation to additional model families is a natural direction for future work.
\bibliographystyle{unsrtnat}
\bibliography{custom}
\clearpage
\appendix

\section{Appendix}
\label{sec:appendix}

\subsection{Dataset Details}
\label{sec:appendix_dataset}
We use the same train/test splits provided by \citet{hubotter2026reinforcement}, summarized in Table~\ref{tab:dataset_stats}.

\begin{table}[H]
\small
\centering
\begin{tabular}{llrrr}
\toprule
Dataset & Source & Train & Test & Total \\
\midrule
Chemistry & SciKnowEval & 1,890 & 210 & 2,100 \\
Physics & SciKnowEval & 720 & 80 & 800 \\
Biology & SciKnowEval & 450 & 50 & 500 \\
Materials & SciKnowEval & 841 & 94 & 935 \\
Tool Use & ToolAlpaca & 4,046 & 68 & 4,114 \\
\bottomrule
\end{tabular}
\caption{Dataset statistics.}
\label{tab:dataset_stats}
\end{table}

\subsection{Hyperparameters}
\label{sec:appendix_hyperparameters}

We summarize hyperparameters in \cref{tab:hyperparameters}.

\begin{table}[H]
\small
% \renewcommand{\arraystretch}{1.2}
% \resizebox{\textwidth}{!}{
\centering
\setlength{\tabcolsep}{3.5pt}
\begin{tabular}{llll}
    \toprule
    \textbf{Parameters} &    \\
    \midrule
    \textbf{General} &  \\
    Model & Qwen/Qwen3-8B   \\
     & allenai/Olmo3-7B-Instruct  \\
    Thinking & False \\
    \midrule
    \textbf{Data} & \\
    Max. prompt length & 2048  \\
    Max. response length  & 2048 \\
    & 8192 (GRPO) \\

    \midrule
    \textbf{Batching}  & \\
    Question batch size  & 32  \\
    Mini batch size & 32  \\
    Number of rollouts & 8 \\
    \midrule
    \textbf{Rollout}  & \\
    Inference engine & vllm   \\
    Temperature  & 1.0   \\

    \midrule
    \textbf{Validation}  &   \\
    Number of rollouts & 16 \\
    Temperature  & 0.6  \\
    Top-$p$  & 0.95   \\
    \midrule
    \textbf{Loss}  & \\
    Top-$K$ distillation  & 100 \\
    Distillation divergence & Jensen–Shannon \\
    Teacher-EMA update rate & 0.05 \\
    $\epsilon$-high & 0.28 (GRPO)\\
    KL coefficient ($\lambda$) & 0 (GRPO) \\
    Normalize advantages  & True (GRPO)\\
     & False (GRPO w/o Norm)\\
    % Rollout importance sampling clip & 2 \\
    \midrule
    \textbf{Training}  & \\
    Optimizer  & AdamW \\
    Learning rate  & $1 \times 10^{-5}$ (constant)  \\
    Warmup steps & 10  \\
    Weight decay  & 0.01  \\
    Gradient Clip Norm & 1.0   \\
    \bottomrule
\end{tabular}
% }
\caption{Hyperparameters.}
\label{tab:hyperparameters}
\end{table}

\subsection{Per-Token Advantage under Jensen--Shannon Divergence}
\label{app:jsd_advantage}

Let $\pi_S := \pi_\theta(\cdot \mid x, y_{i,<t})$ and
$\pi_T := \pi_\theta(\cdot \mid x, f_i, y_{i,<t})$ denote the student and teacher
distributions, and $M := \frac{1}{2}(\pi_S + \pi_T)$ their mixture. The
Jensen--Shannon divergence is
\begin{equation}
    D_{\mathrm{JSD}}(\pi_S \| \pi_T)
    = \frac{1}{2} D_{\mathrm{KL}}(\pi_S \| M)
    + \frac{1}{2} D_{\mathrm{KL}}(\pi_T \| M).
\end{equation}
Taking the gradient with respect to $\theta$ (with stop-gradient on $\pi_T$),
the mixture $M=\tfrac{1}{2}(\pi_S+\pi_T)$ depends on $\pi_S$, so both KL terms
contribute. 

To derive the gradient, we first expand the divergence as a sum over the vocabulary $y$ (omitting the conditioning variables for brevity):
\begin{equation}
    D_{\mathrm{JSD}} = \sum_y \left[ \frac{1}{2} \pi_S(y) \log \frac{\pi_S(y)}{M(y)} + \frac{1}{2} \pi_T(y) \log \frac{\pi_T(y)}{M(y)} \right].
\end{equation}

 \begin{figure*}[t]
      \centering
      \includegraphics[width=\linewidth]{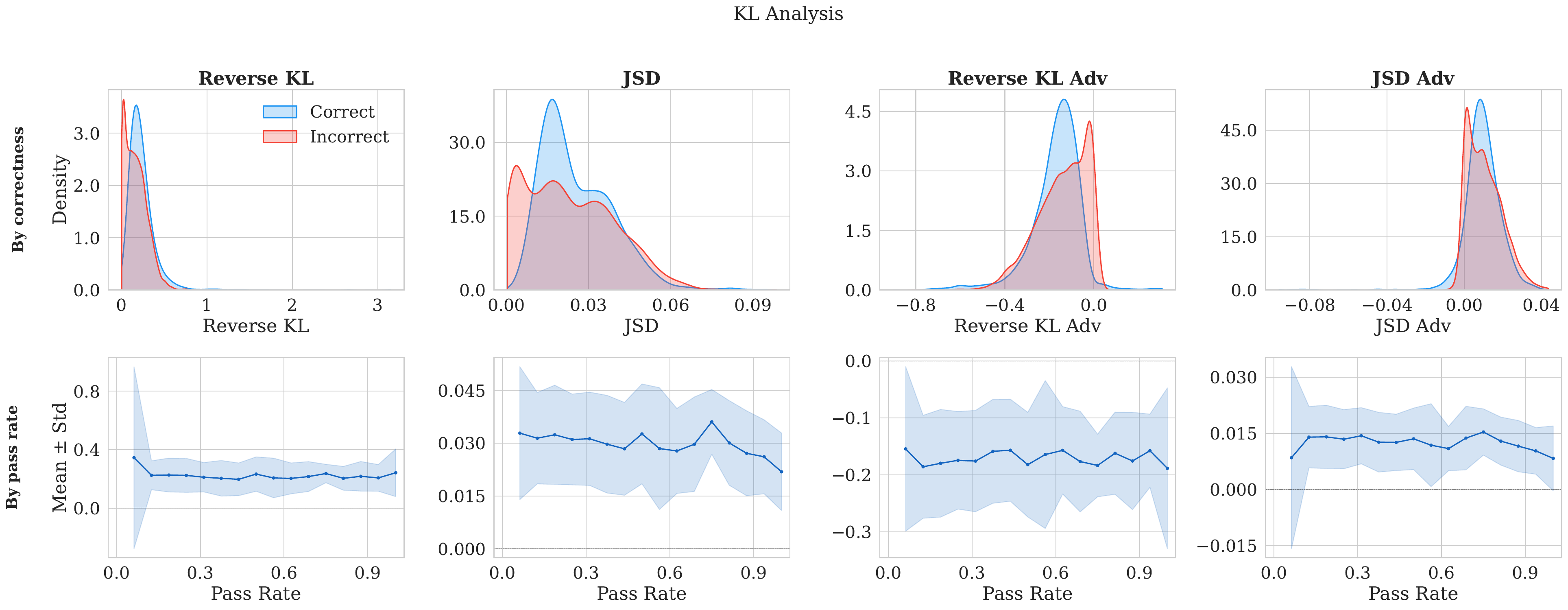}
      \caption{
          KL-based analysis of the base model (\texttt{Qwen3-8B}).
          \textbf{Top row}: kernel density estimates of each metric, stratified by whether the individual rollout is correct (blue) or incorrect (red).
          \textbf{Bottom row}: per-metric mean $\pm$ one standard deviation as a function of the question-level pass rate
          (fraction of correct rollouts among 16, ranging from 0.0625 to 1).
      }
      \label{fig:kl_analysis}
  \end{figure*}
  
We compute the partial derivative of the divergence with respect to the student probability $p = \pi_S(y)$. Let $q = \pi_T(y)$, and the mixture distribution be $m = M(y) = \frac{p+q}{2}$. Applying the product and chain rules, and noting that $\frac{\partial m}{\partial p} = \frac{1}{2}$, we differentiate the terms inside the summation:
\begin{equation}
    \frac{\partial}{\partial p} \left( \frac{1}{2} p \log \frac{p}{m} + \frac{1}{2} q \log \frac{q}{m} \right).
\end{equation}

Evaluating the derivative term by term yields:
\begin{enumerate}
    \item $\frac{\partial}{\partial p} \left( \frac{1}{2} p \log p \right) = \frac{1}{2} (\log p + 1)$
    \item $\frac{\partial}{\partial p} \left( -\frac{1}{2} p \log m \right) = -\frac{1}{2} \log m - \frac{p}{4m}$
    \item $\frac{\partial}{\partial p} \left( \frac{1}{2} q \log q \right) = 0$
    \item $\frac{\partial}{\partial p} \left( -\frac{1}{2} q \log m \right) = -\frac{q}{4m}$
\end{enumerate}

Summing these partial derivatives results in:
\begin{equation}
\begin{aligned}
        \frac{\partial D_{\mathrm{JSD}}}{\partial p} 
        &= \frac{1}{2} \log \frac{p}{m} + \frac{1}{2} - \frac{p + q}{4m}
        \\&= \frac{1}{2} \log \frac{p}{m} + \frac{1}{2} - \frac{2(p + q)}{4(p+q)}
        \\&= \frac{1}{2} \log \frac{p}{m} + \frac{1}{2} -  \frac{1}{2}
        \\&= \frac{1}{2} \log \frac{p}{m}.
\end{aligned}
\end{equation}
Thus:
\begin{equation}
\begin{aligned}
    \nabla_\theta D_{\mathrm{JSD}} &= \sum_y \frac{\partial D_{\mathrm{JSD}}}{\partial \pi_S(y)} \nabla_\theta \pi_S(y)
    \\&= \sum_y \frac{1}{2} \log \frac{\pi_S(y)}{M(y)} \nabla_\theta \pi_S(y)
\end{aligned}
\end{equation}

Using the log-derivative trick, $\nabla_\theta \pi_S(y) = \pi_S(y) \nabla_\theta \log \pi_S(y)$, we can rewrite the gradient as:
\begin{equation}
\begin{aligned}
    \nabla_\theta D_{\mathrm{JSD}} &= \sum_y \pi_S(y) \left[ \frac{1}{2} \log \frac{\pi_S(y)}{M(y)} \right] \nabla_\theta \log \pi_S(y) \\
    &= \mathbb{E}_{y \sim \pi_S} \left[ \frac{1}{2} \log \frac{\pi_S(y)}{M(y)} \nabla_\theta \log \pi_S(y) \right].
\end{aligned}
\end{equation}

The per-token advantage ($A(y)$) is given as:
\begin{equation}
    A_{\mathrm{JSD}}(y) = \frac{1}{2} \log \frac{\pi_S(y)}{M(y)}.
\end{equation}

\subsection{KL Analysis}
\label{app:kl_analysis}

To empirically verify that SDPO's distillation loss does not implicitly recover question-level signal modulation, we analyze the distribution of several KL-based metrics as a function of pass rate on the base model (Qwen3-8B).

We sample $n{=}16$ rollouts per question (temperature $0.6$, top-$p$ $0.95$). The teacher distribution is obtained via the SDPO reprompt template. All distributional metrics are computed over the student's top-100 vocabulary indices with an additional tail bucket. We consider four metrics: the reverse KL divergence $\mathrm{KL}(\pi_\theta \| \pi_t)$, the Jensen--Shannon divergence $\mathrm{JSD} = \frac{1}{2}\mathrm{KL}(\pi_\theta \| M) + \frac{1}{2}\mathrm{KL}(\pi_t \| M)$ with $M = \frac{1}{2}(\pi_\theta + \pi_t)$, and their token-level advantage counterparts: $\log\frac{\pi_t(\hat{y}_t)}{\pi_\theta(\hat{y}_t)}$ and $\frac{1}{2}\log\frac{\pi_\theta(\hat{y}_t)}{M(\hat{y}_t)}$, each averaged over response tokens.

As shown in Figure~\ref{fig:kl_analysis} (bottom row), all four metrics remain approximately flat as a function of pass rate. Note that questions with $\hat{p} = 0$ are excluded from this analysis, since the self-teacher requires a correct rollout as the reference demonstration and no such rollout exists when all attempts fail.
% None exhibits the bell-shaped profile characteristic of GRPO's implicit signal modulation, where gradient magnitude peaks at $\hat{p} \approx 0.5$.

\subsection{Learnability Bound under Normalized Rewards}
\label{sec:normalized_bound}

The reverse KL bound of \citet{bae2026online}
(Eq.~\ref{eq:reward_variance}) is derived under the raw
binary reward $r \in \{0,1\}$. Here we extend their CGF
framework (Proposition~6.1 in \citealt{bae2026online}) to
the standardized reward that GRPO's group-relative
normalization implicitly operates on.

\begin{proposition}
\label{prop:normalized_cgf}
Let $r \in \{0,1\}$ with $\mathbb{E}[r] = p \in (0,1)$. Define $\tilde{r} := (r - p)/\sqrt{p(1-p)}$ and the corresponding optimal policy
$\tilde{\pi}^* \propto \pi_{\mathrm{init}} \exp(\tilde{r}/\beta)$.
Then
\begin{equation}
    D_{\mathrm{KL}}\!\left(\pi_{\mathrm{init}} \,\|\,
    \tilde{\pi}^*\right)
    = \frac{1}{2\beta^2}
    + O(\beta^{-3}).
    \label{eq:normalized_kl_appendix}
\end{equation}
\end{proposition}
\begin{proof}
By the CGF identity of \citet{bae2026online},
$D_{\mathrm{KL}}(\pi_{\mathrm{init}} \| \tilde{\pi}^*) = K_{\tilde{r}}(1/\beta)$,
where $K_Z(t) = \sum_{k=2}^{\infty} \kappa_k t^k / k!$ for a zero-mean variable~$Z$.
The standardized reward takes values $\sqrt{(1-p)/p}$ with probability~$p$ and $-\sqrt{p/(1-p)}$ with probability~$1-p$, giving:
\begin{equation}
\begin{split}
K_{\tilde{r}}\!\left(\tfrac{1}{\beta}\right)
&= \frac{\kappa_2}{2\beta^2}
+ O(\beta^{-3})
\\&= \frac{\mathrm{Var}[\tilde{r}]}{2\beta^2}
+ O(\beta^{-3})
\\&= \frac{p \cdot \frac{1-p}{p} + (1-p) \cdot \frac{p}{1-p}}{2\beta^2}
+ O(\beta^{-3})
\\&= \frac{1}{2\beta^2}
+ O(\beta^{-3}).
\end{split}
\end{equation}
\end{proof}
Comparing with the raw-reward bound $D_{\mathrm{KL}} \geq p(1-p)/(2\beta^2)$ of \citet{bae2026online}, the leading term under normalized rewards becomes $1/(2\beta^2)$---a constant independent of~$p$. This reveals that GRPO's group-relative normalization acts as an implicit difficulty-aware mechanism: by dividing the reward by $\sqrt{p(1-p)}$, it equalizes the leading-order learnability across all non-degenerate questions ($0 < p < 1$), absorbing the $p(1-p)$ factor that appeared in the raw-reward bound. Independently, the per-step gradient magnitude is exactly $\sqrt{p(1-p)}$ by the algebraic identity in Eq.~\ref{eq:sweet_spot_magnitude}, supporting $\alpha = 1/2$ as the scale-consistent exponent.

\begin{figure*}[p] 
    \centering
    
    \begin{subfigure}{\linewidth}
        \centering
        \includegraphics[width=0.95\linewidth]{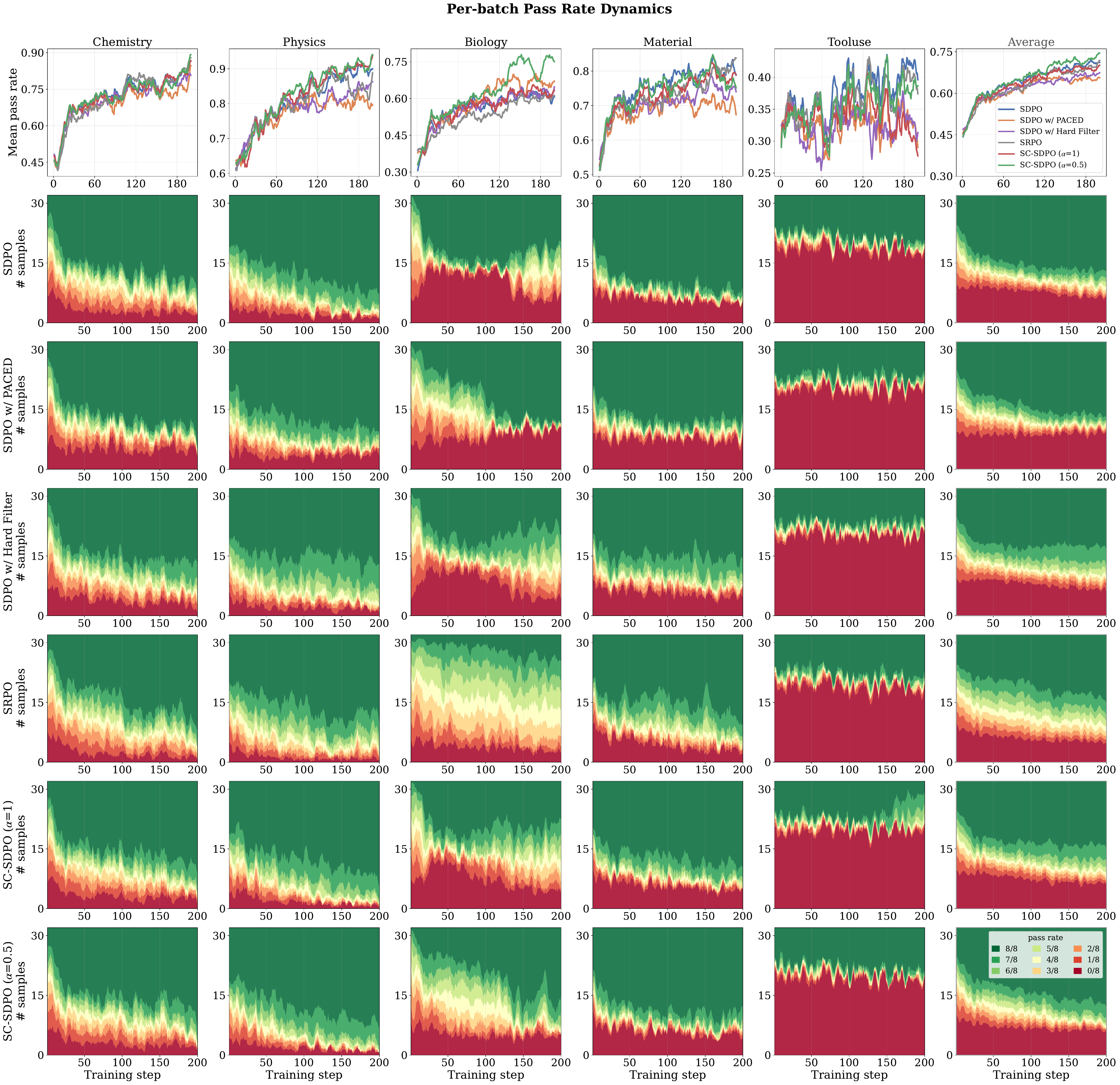}
        \caption{
            Per-batch pass rate dynamics over 200 training steps across five datasets and their average (rightmost column).
            \textbf{Row 1:} Mean pass rate per training batch for each method, smoothed with a sliding window.
            \textbf{Rows 2--7:} Stacked area charts showing the sample count distribution across nine discrete pass rate levels ($0, \tfrac{1}{8}, \ldots, 1$) at each step, color-coded from red (pass rate $= 0$, all rollouts incorrect) to green (pass rate $= 1$, all rollouts correct), smoothed with a sliding window.
        }
        \label{fig:passrate_dynamics}
    \end{subfigure}

    \vspace{1.5em}

    \begin{subfigure}{\linewidth}
        \centering
        \includegraphics[width=0.8\linewidth]{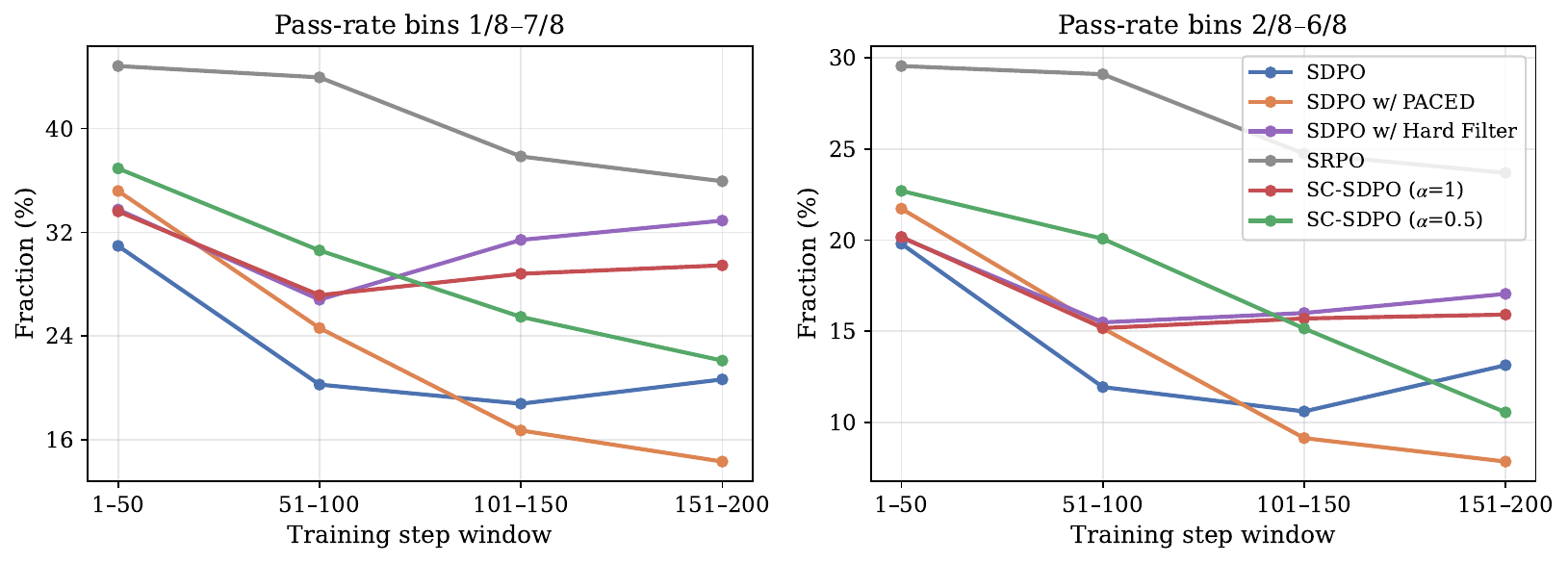}
        \caption{
            Fraction of training samples in pass-rate bins 1/8–7/8 (left) and bins 2/8–6/8 (right) per 50-step training window, averaged across five datasets. 
        }
        \label{fig:passrate_stats}
    \end{subfigure}

\end{figure*}

\subsection{Pass Rate Analysis}
\label{app:passrate_analysis}
Figure~\ref{fig:passrate_dynamics} tracks the per-batch pass rate distribution over 200 training steps. On most datasets, the model's growing competence gradually converts intractable questions ($\hat{p} = 0$, red) into solvable ones, expanding the green (mastered) region over time. On Biology, \ourmethod{} ($\alpha = 0.5$) maintains a visibly broader band of intermediate pass rates ($\hat{p} \in [1/8, 7/8]$) compared to the SDPO baseline, indicating a larger effective learnable zone throughout training. ToolUse presents a distinct pattern: pass rates for all methods remain heavily concentrated at the extremes ($\hat{p} \approx 0$ or $\hat{p} \approx 1$), leaving a narrow intermediate region and limited room for difficulty-aware weighting to differentiate methods.

Figure~\ref{fig:passrate_stats} quantifies the fraction of training samples in the intermediate pass-rate bins across 50-step windows. SRPO maintains the highest intermediate fraction throughout training in both the broader ($\hat{p} \in [1/8, 7/8]$) and narrower ($\hat{p} \in [2/8, 6/8]$) ranges. This is expected given SRPO's routing mechanism: by directing incorrect rollouts with correct siblings to the SDPO branch and all others to GRPO, it effectively preserves a large pool of mixed-outcome groups rather than resolving them toward mastery. However, as shown in the main results, this higher intermediate retention does not fully translate into superior accuracy, suggesting that sustaining a large sweet spot is necessary but not sufficient. Among non-routing methods, \ourmethod{} ($\alpha = 0.5$) achieves the highest intermediate fraction in the early phase of training (steps 1--100), where the learning frontier is broadest and difficulty-aware weighting has the greatest impact. As training progresses and more questions are mastered, the intermediate fraction naturally declines across all methods; \ourmethod{} converges smoothly, consistent with the implicit curriculum effect of our batch-adaptive design.

\subsection{Prompt Templates}
\label{app:prompt}

   We adopt the same prompt templates as \citet{hubotter2026reinforcement}.

  \paragraph{Science QA Prompt (SciKnowEval)}

  Science QA tasks use a fixed system prompt combined with a per-question user message.

  \begin{tcolorbox}[title=System Prompt, colback=gray!5]
  \texttt{Given a question and four options, please select the right answer.
  Respond in the following format:}

  \texttt{<reasoning>}\\
  \texttt{...}\\
  \texttt{</reasoning>}\\
  \texttt{<answer>}\\
  \texttt{...}\\
  \texttt{</answer>}

  \texttt{For the answer, only output the letter corresponding to the correct
  option (A, B, C, or D), and nothing else. Do not restate the answer text.
  For example, if the answer is ``A'', just output:}

  \texttt{<answer>}\\
  \texttt{A}\\
  \texttt{</answer>}
  \end{tcolorbox}

  \begin{tcolorbox}[title=User Message Template, colback=gray!5]
  \texttt{\{question\}}

  \texttt{A: \{choice\_A\}}\\
  \texttt{B: \{choice\_B\}}\\
  \texttt{C: \{choice\_C\}}\\
  \texttt{D: \{choice\_D\}}

  \texttt{Please reason step by step.}
  \end{tcolorbox}

  \begin{tcolorbox}[title=Teacher Context — With Solution, colback=blue!5]
  \textbf{[User]}

  \texttt{\{problem\}}

  \texttt{Please reason step by step, and put your final answer within \textbackslash boxed\{\}.}

  \texttt{Correct solution:}

  \texttt{\{successful\_previous\_attempt\}}

  \texttt{Correctly solve the original question.}

  \medskip
  \textbf{[Assistant]}

  \texttt{\{student\_response\ $\hat{y}$\}}
  \end{tcolorbox}
  
  \paragraph{Tool Use Prompt}

  Tool use tasks have no separate system prompt. The entire context is
  contained in a single user message structured as follows:

  \begin{tcolorbox}[title=Tool Use Prompt Template, colback=gray!5]
  \texttt{Your task is to answer the user's question using available tools.}\\
  \texttt{You have access to the following tools:}

  \texttt{Name: \{tool\_name\}}\\
  \texttt{Description: \{tool\_description\}}\\
  \texttt{Documentation:}\\
  \texttt{\{method\_name\}: \{method\_description\}}\\
  \texttt{Parameters: \{parameters\_json\}}\\
  \texttt{Output: \{output\_description\}}\\
  \texttt{[\ldots\ additional tools \ldots]}

  \medskip
  \texttt{Use the following format:}

  \texttt{Thought: you should always think about what to do}\\
  \texttt{Action: the action to take, should be one of the tool names.}\\
  \texttt{Action Input: the input to the action, must be in JSON format.}\\
  \texttt{All of the action input must be realistic and from the user.}

  \medskip
  \texttt{Begin!}\\
  \texttt{Question: \{user\_question\}}
  \end{tcolorbox}

  \noindent The expected response follows the \texttt{Thought / Action /
  Action Input} format. The verifier checks two conditions jointly:
  (1) the predicted action names match the ground-truth action names
  (multiset equality), and (2) the merged JSON parameters across all
  \texttt{Action Input} blocks exactly match the ground-truth parameters.
  Both conditions must hold for a reward of~$1$ (correct); otherwise the reward
  is~$0$ (incorrect).

  \noindent When SDPO self-distillation is applied to tool use, the
  feedback section in the teacher context is populated with
  a natural-language description of the mismatch, e.g.:

  \begin{tcolorbox}[title=Feedback Example (Tool Use), colback=orange!5]
  \texttt{Action inputs mismatch: predicted \{``method'': ``GET'', \ldots\},}\\
  \texttt{expected \{``method'': ``POST'', \ldots\}}
  \end{tcolorbox}

\end{document}